\documentclass[11pt]{article}

\usepackage[margin=1in]{geometry}
\usepackage{amsmath}
\usepackage{amssymb}
\usepackage{amsthm}
\usepackage{graphicx}
\usepackage{booktabs}
\usepackage{hyperref}
\usepackage{natbib}
\usepackage{authblk}

\newtheorem{theorem}{Theorem}
\newtheorem{proposition}{Proposition}
\newtheorem{definition}{Definition}
\newtheorem{property}{Property}
\newtheorem{corollary}{Corollary}
\newtheorem{remark}{Remark}

\title{FRACTAL: A Self-Similar Multi-Agent Orchestration Architecture\\
with Formal Communication, Convergence, and Depth Bounds}
\author{Ohad Dan \and Abhinav Taliyan}
\date{}

\begin{document}
\maketitle

\begin{abstract}
We present FRACTAL, a self-similar multi-agent architecture for recursive goal decomposition, in which the same five-component organizational pattern --- Orchestrator, Ensemble, Knowledge Ledger, Decision Maker, Meta-Observer --- recurs at every level of a decomposition tree. We make four formal contributions: (1) a communication bound showing intra-node cost is linear in ensemble size and round count ($O(kr)$) under a two-phase blind-proposal/stress-test protocol, versus quadratic cost for free-form debate, with an explicit, stated assumption on fixed evaluator count; (2) a two-mode orchestration criterion separating generative and evaluative phases via a measurable diversity-saturation threshold; (3) a supermajority convergence guarantee with explicit probability bounds at stated agreement levels, and a formal depth bound under an operationalized goal-complexity measure; (4) a Bayesian inter-layer escalation mechanism unifying hysteresis and contract-breach signaling under one posterior update. We illustrate the architecture's execution qualitatively on a worked example and a multi-layer case study; we do not report quantitative performance claims in this paper, and we explicitly discuss the formal and inductive-bias limitations of the model, including an unimplemented but architecture-compatible direction toward bounded self-improvement. We use ``fractal'' as an architectural metaphor for self-similarity, not a measured mathematical property. Empirical validation of the architecture's comparative performance is left to future work.
\end{abstract}


\section{Introduction}

Consider a startup founder who asks an AI system: ``Design a go-to-market strategy for my AI code review tool.'' A single LLM produces a competent but generic plan. A multi-agent debate yields richer perspectives --- but at quadratic communication cost, with agents that converge prematurely on shared assumptions, and no principled way to decompose the macro goal (strategy) into sub-goals (pricing, positioning, risk mitigation) and sub-sub-goals (specific channel tactics). What if the same organizational structure that produced the strategy --- a small team of diverse experts, debating, then deciding, then delegating --- could recursively spawn itself at each level of detail?

This intuition motivates FRACTAL (Fractal Recursive Agents Coordinating Through Autonomous Layers). We use ``fractal'' here as an architectural metaphor for the recursive self-similarity of the node structure (Definition~1) --- the same five-component pattern recurring at every level of the decomposition tree --- not as a claim of measured fractal dimension or scale-invariant statistics on any real system's output.

LLM-based multi-agent systems have demonstrated improvements over single-agent inference on reasoning \citep{du2024improving}, code generation \citep{hong2024metagpt}, and planning \citep{yao2023tree}. But as these systems grow, three challenges emerge:

\begin{enumerate}
\item \textbf{Communication cost.} Free-form debate among $k$ agents over $r$ rounds requires $O(k^2 r)$ messages (each agent responds to every other). This quadratic scaling limits ensemble sizes and incentivizes shallow interaction.
\item \textbf{Premature convergence.} Shared training data creates ``synthetic diversity'' that collapses under mutual evaluation \citep{liang2024encouraging}. Human creativity research is consistent with this: continuous cross-evaluation during ideation degrades output \citep{bernstein2018intermittent}.
\item \textbf{Unbounded recursion.} Hierarchical decomposition can proceed indefinitely. The options framework \citep{sutton1999between} provides temporal abstraction but is empirically limited to two layers in practice \citep{eppe2022intelligent}. No existing LLM multi-agent framework provides formal depth bounds.
\end{enumerate}

FRACTAL addresses all three through self-similarity: every node in the decomposition tree --- whether handling the macro-strategic question (``what market to enter'') or the micro-tactical one (``which pricing page A/B test to run'') --- instantiates the same five-component structure and obeys the same communication protocol. Bounds on communication, convergence, and depth are derived once and hold at every layer.

\paragraph{Running example.} Throughout this paper, we trace FRACTAL's operation on the startup strategy task above, as an illustrative walkthrough of protocol execution rather than a quantitative evaluation. Section~5 shows how Divergent mode generates four qualitatively different strategy angles (product-led growth, accuracy-first positioning, risk-bounded design, and a complete reframing toward ``technical debt quantification''); Section~6 shows how the Convergent critique phase identifies that the reframing is a feature rather than a standalone product --- an insight absent from every individual proposal.

\paragraph{Contributions.} We formalize FRACTAL as a rooted tree of identical nodes (Section~3), derive communication bounds (Section~4), introduce two-mode orchestration with stability guarantees (Section~5), present convergence and depth bounds (Section~6), present Bayesian inter-layer escalation (Section~7), and characterize the formal bounds a future bounded self-improvement extension would need to respect (Section~8). We illustrate the formalism qualitatively (Sections~5, 6, and Appendix~A); empirical validation of the architecture's comparative performance against single-agent or flat-debate baselines is left to future work (Section~9).

\section{Related Work}

\paragraph{Multi-agent debate and collaboration.} Multi-Agent Debate (MAD) \citep{liang2024encouraging} addresses the degeneration-of-thought problem through adversarial discussion but operates in a flat topology without hierarchy or communication bounds. \citet{du2024improving} show that debate improves factuality; DyLAN \citep{liu2024dylan} introduces dynamic agent selection with an unsupervised importance-scoring metric and a BLEU-based consistency check used for early-stopping. MetaGPT \citep{hong2024metagpt} achieves strong results via sequential SOPs with structured artifacts, demonstrating that coordination structure matters more than agent count. AutoGen \citep{wu2023autogen} provides configurable conversation topologies. ChatDev \citep{qian2024chatdev} and CAMEL \citep{li2023camel} explore role-playing collaboration. None provides formal communication bounds, convergence guarantees, or recursive self-similarity.

\paragraph{Hierarchical planning and reinforcement learning.} The options framework \citep{sutton1999between}, MAXQ \citep{dietterich2000maxq}, and feudal architectures provide temporal abstraction but are limited to two layers in practice \citep{eppe2022intelligent}. Tree of Thoughts \citep{yao2023tree} enables hierarchical search but uses self-evaluation rather than multi-agent deliberation. FRACTAL replaces learned option policies with LLM-generated contracts, aiming at arbitrary depth through formal depth bounds rather than empirical scaling.

\paragraph{Self-improving AI systems.} Self-Refine \citep{madaan2023selfrefine} achieves inference-time improvement through iterative feedback; Reflexion \citep{shinn2023reflexion} stores verbal reinforcement in episodic memory; ADAS \citep{hu2024adas} searches over code-defined agent architectures; EvoAgent \citep{yuan2025evoagent} evolves multi-agent ensembles; Voyager \citep{wang2023voyager} builds skill libraries for lifelong learning; FunSearch \citep{romeraparedes2024mathematical} pairs LLMs with evaluators to search program space. The Agentic Neural Network framework \citep{ma2025selfevolving} conceptualizes multi-agent systems as neural network layers with textual backpropagation. Section~8 discusses what formal bounds (contracts, not structure) a FRACTAL extension in this space would need to respect --- not a currently implemented self-modification mechanism.

\paragraph{Cognitive foundations.} The divergent--convergent mode separation draws loosely on the neuroscience of creativity, where the Default Mode Network (idea generation) and Executive Control Network (evaluation) are often described as operating in anti-correlated phases. \citet{bernstein2018intermittent} demonstrate that intermittent breaks in human group interaction improve collective intelligence --- the inspiration for our minimum dwell-time constraint (Proposition~10).

\paragraph{Practitioner frameworks.} Engineering-oriented systems (LangGraph, CrewAI, LlamaIndex Workflows) provide flexible agent orchestration with rich tool ecosystems but without formal guarantees on communication cost, convergence, or recursion depth. FRACTAL contributes principled bounds and theoretical analysis rather than competing on ease-of-use or ecosystem maturity; the formal properties could inform design decisions within these frameworks.

\section{The Fractal Node}

Informally, a Fractal Node is a small team of diverse AI experts who brainstorm freely (Divergent mode), then critically evaluate proposals (Convergent mode), then decide on a path forward and delegate sub-goals to child teams --- each structurally identical. Communication between levels flows only through compressed ``contracts,'' never raw conversation (Figure~1). We now formalize this structure.

\begin{definition}[Fractal Node]
A Fractal Node at depth $d$ is the 5-tuple $N_d = (O_d, E_d, L_d, D_d, M_d)$:
\begin{itemize}
\item $O_d$: Orchestrator --- a finite automaton with states $Q_O = \{q_{div}, q_{conv}\}$ and transition function $O_d: (Q_O \times \Sigma) \to (Q_O \times \{\mathrm{route}(e_i), \mathrm{terminate}, \mathrm{query}, \mathrm{switch}\})$
\item $E_d = \{e_1, \ldots, e_k\}$: Ensemble of $k \in [3,5]$ persona agents, each $e_i = (p_i, \beta_i, T_i, \phi_i)$ with domain prompt $p_i$, cognitive bias vector $\beta_i \in \mathbb{R}^{|B|}$, tool access $T_i$, and output schema $\phi_i$
\item $L_d = (F_d, C_d, R_d, D_d)$: Knowledge Ledger of verified facts, constraints, risks, and discarded alternatives
\item $D_d: (L_d, S_d) \to (a_d^*, \{X_{d+1}^1, \ldots, X_{d+1}^n\})$: Decision Maker
\item $M_d: \mathrm{Trace}(N_d) \to R_{audit}$: Meta-Observer (\textbf{read-only})
\end{itemize}
\end{definition}

\begin{definition}[FRACTAL System]
A FRACTAL system is a rooted tree $T = (V, E)$ where $V = \bigcup_{d=0}^{d_{max}} V_d$, each $v \in V$ is a Fractal Node, and each edge $(v_d^i, v_{d+1}^j) \in E$ carries an immutable Execution Contract $X_{d \to d+1}^{i \to j}$.
\end{definition}

\paragraph{Self-Similarity Invariant.} The structure of every $N_d$ is identical regardless of $d$. Only the parameters --- personas, domain prompts, clock speed, tools --- vary.

\begin{remark}[Why self-similarity?]
Self-similarity is not laziness --- it is the mechanism that enables formal analysis. Because every node obeys the same communication protocol, Theorem~6 (token bound) applies uniformly at all depths without per-layer derivation. Because every node uses the same stopping condition, Theorem~12 (convergence) holds at every layer. This composability is lost if nodes at different depths have different internal structures. The tradeoff: self-similarity imposes a minimum overhead even for simple sub-goals. Depth-adaptive node sizing (varying $k$ with task complexity) is a natural extension we leave to future work.
\end{remark}

\begin{definition}[Cognitive Diversity Constraint]
For ensemble $E_d = \{e_1, \ldots, e_k\}$, define angular distance $\mathrm{div}(e_i, e_j) = 1 - \langle \beta_i, \beta_j \rangle / (\|\beta_i\| \|\beta_j\|)$. The ensemble satisfies diversity threshold $\delta$ iff $\min_{i \neq j} \mathrm{div}(e_i, e_j) \geq \delta$. We target $\delta \approx 0.4$ as a design choice of this architecture, not derived from prior empirical work.
\end{definition}

\emph{(On attribution: an earlier draft of this definition cited DyLAN's \citep{liu2024dylan} ``finding that moderate contentiousness optimizes performance'' as the basis for $\delta \approx 0.4$. We independently checked DyLAN's full text and found no mention of cosine similarity, bias vectors, or an angular diversity threshold anywhere in that paper; DyLAN's only numeric threshold is a BLEU score of 0.9 used for an early-stopping consistency check across rounds, unrelated to ensemble diversity at team-formation time. We have removed that citation from this definition accordingly --- $\delta \approx 0.4$ is this architecture's own design choice, not a value derived from DyLAN or any other prior work we are aware of.)}

\begin{definition}[Execution Contract]
An Execution Contract is $X = (g, \Theta, \tau_{exp}, G, T_{child}, P_{esc}, F_{cls})$ where $g$ is the sub-goal, $\Theta$ quantitative bounds, $\tau_{exp}$ temporal expiration, $G$ guardrails, $T_{child}$ tool manifest, $P_{esc}$ escalation triggers, and $F_{cls}$ falsification clauses specifying observable conditions that would invalidate the contract's premises.
\end{definition}

\paragraph{Information Inheritance.} A child node receives only its contract --- never the parent's raw deliberation history. Communication is one-way (parent$\to$child) and one-time (at spawn). Upward communication is zero except on CONTRACT\_BREACH.

\begin{figure}[t]
\centering
\includegraphics[width=0.85\textwidth]{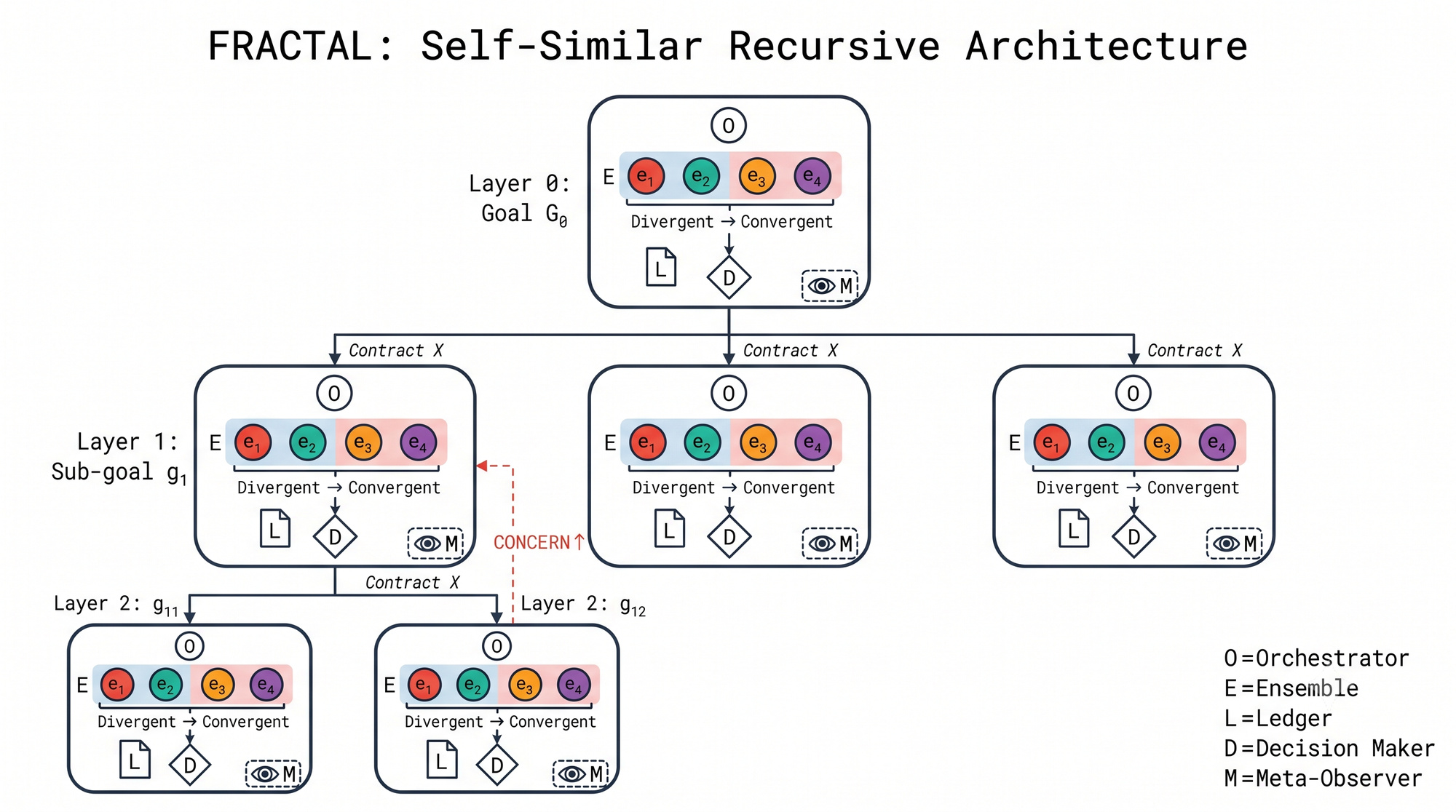}
\caption{FRACTAL's formal node structure (Definition~1), shown schematically --- a diagram of the model, not a measured system. Self-similar recursive tree where every node (regardless of depth) instantiates the same 5-component structure: Orchestrator (O), Ensemble (E), Knowledge Ledger (L), Decision Maker (D), Meta-Observer (M). Downward arrows carry immutable Execution Contracts; upward dashed arrows indicate CONCERN escalation. Sibling isolation (no lateral communication) is enforced by construction.}
\end{figure}

\section{Communication Complexity}

\begin{theorem}[Two-Phase Communication Bound]
Under the two-phase protocol (Blind Proposal $\to$ Stress Test), run for $r$ divergent rounds before a single convergent/synthesis step, intra-node token cost is:
\begin{equation}
C_{2P}(N_d, r) = r \cdot k \cdot T_p + m \cdot c \cdot T_s + T_{synth}
\end{equation}
where $T_p$ is tokens per proposal, $m = 2$ is the number of evaluator personas (held fixed, independent of ensemble size $k$ --- see the assumption discussion below), $c = 3$ critiques per proposal, $T_s$ tokens per critique, and $T_{synth}$ the Decision Maker synthesis. Under free-form debate with $r$ rounds:
\begin{equation}
C_{debate}(N_d, r) = k \cdot r \cdot T_{turn} + \binom{k}{2} \cdot r \cdot T_{resp}
\end{equation}
Equation~1 is linear in $k$ and $r$ (genuinely $O(kr)$); Equation~2 is quadratic in $k$ ($O(k^2 r)$).
\end{theorem}

\noindent\emph{Note on the fixed-evaluator assumption.} The evaluator count $m$ is held fixed at 2 regardless of ensemble size $k$ --- this is a stated architectural choice (Mode 2, Section~5, names exactly two evaluator personas: Structural Purist and Friction Engine), not a free parameter that happens to scale favorably. If evaluators scaled with $k$ instead (every one of $k$ agents evaluating every proposal), the Phase-2 term would be $k \cdot c \cdot T_s$ and, combined with $k$ proposers, the comparison to debate would depend on the relative constants rather than the clean $O(kr)$ vs.\ $O(k^2r)$ separation stated here. We make this assumption explicit rather than letting it stay implicit in the formula, per Theorem~6's role in Corollary~8 and the comparison in Figure~2.

\begin{proof}
The two-phase protocol eliminates the $\binom{k}{2}$ pairwise interaction term by routing all communication through the Orchestrator. Over $r$ divergent rounds, $k$ agents generate proposals independently each round ($r \cdot k \cdot T_p$ total). After the mode switch (Section~5), the $m = 2$ fixed evaluators each produce $c$ critiques per proposal once ($m \cdot c \cdot T_s$), followed by one synthesis step ($T_{synth}$). No agent-to-agent direct communication occurs. Free-form debate requires each pair of agents to exchange messages in each of $r$ rounds, yielding the quadratic term.
\end{proof}

\paragraph{Worked numeric example (illustrative, not measured).} To make the comparison concrete and checkable by hand, fix one self-consistent set of token constants: $T_p = T_{turn} = 500$, $T_s = T_{resp} = 300$, $T_{synth} = 500$, $c = 3$, $m = 2$, $r = 3$. Table~1 shows the resulting costs.

\begin{table}[h]
\centering
\begin{tabular}{@{}lllll@{}}
\toprule
$k$ & $C_{2P}$ (Eq.\ 1) & $C_{debate}$ (Eq.\ 2) & $C_{2P}/C_{debate}$ & Reduction \\
\midrule
3 & $3{,}500+1{,}800+500=6{,}800$ & $4{,}500+2{,}700=7{,}200$ & $0.944$ & $1.06\times$ \\
4 & $6{,}000+1{,}800+500=8{,}300$ & $6{,}000+5{,}400=11{,}400$ & $0.728$ & $1.37\times$ \\
5 & $7{,}500+1{,}800+500=9{,}800$ & $7{,}500+9{,}000=16{,}500$ & $0.594$ & $1.68\times$ \\
\bottomrule
\end{tabular}
\caption{Worked numeric example under one self-consistent, explicitly stated set of token constants (illustrative, not measured). FRACTAL is cheaper than debate across the entire stated range $k\in[3,5]$ under these constants.}
\end{table}

Under these constants and the fixed-$m$ assumption, FRACTAL is cheaper than free-form debate across the entire stated range $k \in [3,5]$, with the advantage widening as $k$ grows --- the expected qualitative behavior of a linear-vs-quadratic comparison. We do not claim these specific token values are measured; they are one illustrative, internally consistent choice, stated explicitly here (rather than left implicit, as in an earlier draft) so a reader can verify Figure~2 and the ratios above by hand with different constants of their own choosing. We also do not claim a fixed universal reduction factor (e.g., ``$15\times$'') independent of the chosen constants and $k$: the reduction ratio in Theorem~6 is $O(1/(kr))$ asymptotically, but its concrete numeric value at any specific $k$ depends on the token constants, as the table shows.

\begin{figure}[t]
\centering
\includegraphics[width=0.75\textwidth]{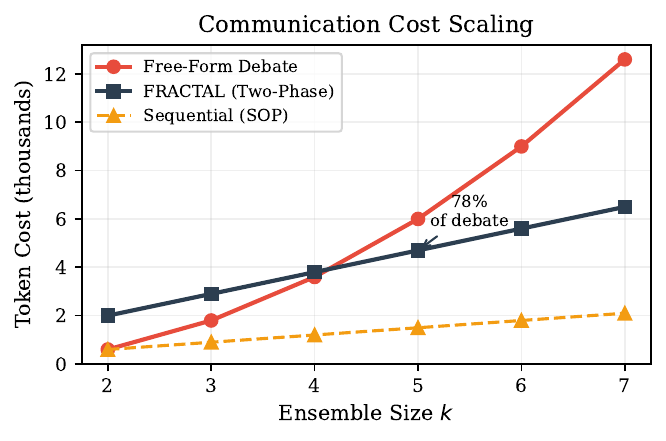}
\caption{Communication cost as predicted by Theorem~6's formula (Equation~1 vs.\ Equation~2) under the token constants stated in the worked numeric example ($T_p=T_{turn}=500$, $T_s=T_{resp}=300$, $T_{synth}=500$, $c=3$, $m=2$, $r=3$), not measured experimentally. FRACTAL is cheaper than free-form debate for all $k\in[3,5]$ under these constants; see Theorem~6 for the general $O(kr)$ vs.\ $O(k^2r)$ comparison.}
\end{figure}

\begin{property}[Communication Asymmetry]
By construction, in a well-formed FRACTAL system: $B_{down}(d \to d+1) = n_d \cdot |X|$ (once at spawn); $B_{up}(d+1 \to d) = 0$ in normal operation; $B_{lateral}(d,i \leftrightarrow d,j) = 0$ always (sibling isolation). Total system bandwidth is $O(|V| \cdot T_{contract})$ --- linear in the tree size.
\end{property}

\begin{corollary}[Total System Cost]
For uniform branching factor $n$ and depth $d_{max}$:
\begin{equation}
C_{total} = \sum_{d=0}^{d_{max}} n^d \cdot C_{2P}(N_d, r) + \sum_{d=0}^{d_{max}-1} n^{d+1} \cdot T_{contract} = O(n^{d_{max}} \cdot k \cdot r)
\end{equation}
Cost is exponential in depth but linear in ensemble size and round count --- justifying bounded depth (Theorem~14) and moderate ensemble sizes ($k \in [3,5]$).
\end{corollary}

\section{Two-Mode Orchestration}

The ensemble operates in two architecturally separated modes, with mode-switching governed by a measurable saturation criterion rather than fixed round counts.

\paragraph{Mode 1 --- Divergent.} Generator personas (Radical Speculator, Integrative Synthesizer, Frame Challenger) are active; evaluator personas receive no input. Two sub-phases occur: Generation (independent blind proposals, Round~1) and Integration (respond to Frame Challenger's reframing, hybridize, Round~2). The Frame Challenger (Definition~11) is mandatory in this mode.

\emph{Illustrative walkthrough of protocol execution, not a quantitative evaluation.} In the startup strategy task, Divergent mode produces four proposals from personas with non-overlapping information: product-led growth targeting power users (Speculator, who received market sizing data), accuracy-first positioning with published benchmarks (Purist, who received competitive analysis), risk-bounded design addressing security concerns (Friction Engine, who received churn data), and a complete reframing toward ``technical debt quantification'' sold to VPs rather than developers (Frame Challenger, who received adjacent-market failure cases). The Frame Challenger forces the group to defend the ``developer tool'' frame, which ultimately produces a hybrid: the dollar-quantification idea is absorbed as a feature of the developer-facing product rather than adopted as a standalone pivot.

\paragraph{Mode 2 --- Convergent (Muted Democracy).} Evaluator personas (Structural Purist, Friction Engine) are active; generators are muted. Each evaluator produces exactly 3 critiques per proposal. The Decision Maker reviews the flat critique matrix and selects a path.

\emph{Illustrative walkthrough, continued.} Convergent critiques identify that the Frame Challenger's reframing changes the buyer from developers (fast adoption) to VPs (slow enterprise sales), contradicting the need for quick revenue. This critique transforms the reframing from a standalone direction into an embedded feature --- cost-of-escape scoring within a developer-facing product. This insight is absent from all four individual proposals in this walkthrough and emerges only from the adversarial evaluation step.

\begin{definition}[Mode-Switching Criterion]
The Orchestrator switches $q_{div} \to q_{conv}$ when:
\begin{equation}
\frac{\Delta\mathrm{Div}(r)}{\Delta\mathrm{Div}(r-1)} < \tau_{sat} \quad \text{and} \quad r \geq r_{dwell}
\end{equation}
where $\Delta\mathrm{Div}(r)$ is marginal diversity gain at round $r$ and $r_{dwell} = 2$ enforces minimum dwell time.
\end{definition}

\begin{proposition}[Dwell-Time Guarantee]
This is a conditional/parametric result: \emph{if} the system exhibits geometric diversity decay with rate $\gamma < 1$, \emph{then}, for $r_{dwell} \geq 2$, the expected round of mode switch satisfies $E[r^*] \geq r_{dwell}$, ensuring at least one Integration round following Generation. This paper does not estimate $\gamma$ empirically for any implementation; it is a modeling assumption, illustrated with an example value ($\gamma = 0.4$, Theorem~13) for concreteness, not a measured property of any run. The system terminates in finite rounds under the composite stopping condition $\mathrm{STOP} = \mathrm{SuperMajority}(E_d) \vee \mathrm{DR}(L_d, \epsilon) \vee (r \geq r_{max})$.
\end{proposition}

\begin{figure}[t]
\centering
\includegraphics[width=0.75\textwidth]{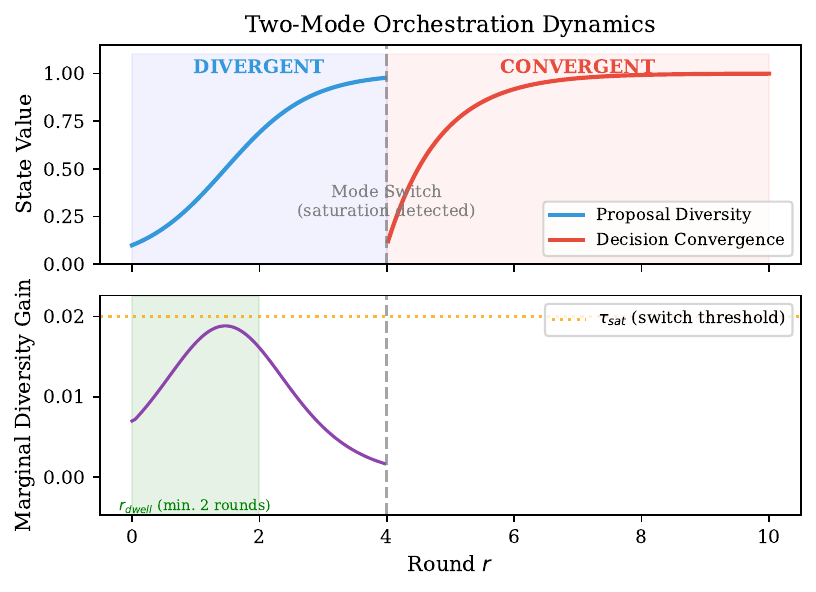}
\caption{Two-mode dynamics as predicted by the model in Definition~9 and Proposition~10 under illustrative parameter values, not measured from any run. Top: proposal diversity grows logistically in Divergent mode; convergence follows in Convergent mode. Bottom: mode switch triggered when marginal diversity gain drops below $\tau_{sat}$, subject to minimum dwell time $r_{dwell} \geq 2$.}
\end{figure}

\begin{definition}[Frame Challenger]
A designated persona in the Divergent phase whose role is to challenge the problem framing itself. Must produce: (1) a specific failure scenario for the current framing; (2) an alternative success mechanism; (3) a concrete modification proposal. Functions as a regularizer against premature convergence from shared-knowledge bias --- loosely analogous to dropout in neural networks, where ``dropping'' the assumed frame forces cross-frame robust solutions.
\end{definition}

\section{Convergence and Depth Bounds}

\begin{theorem}[Supermajority Convergence]
Let $p \in [0.5, 0.9]$ be the per-agent agreement probability (moderate contentiousness regime). For $k = 5$, the probability of supermajority ($\geq 4$ of 5) consensus within $r$ rounds is $P(\text{consensus in } r) = 1 - (1 - P_1)^r$ where $P_1 = \sum_{j=4}^{5} \binom{5}{j} p^j (1-p)^{5-j}$. At $p = 0.7$: $P_1 = \binom{5}{4}(0.7)^4(0.3) + (0.7)^5 = 0.36015 + 0.16807 = 0.528$, yielding $P(\text{consensus in } 3) = 1 - (0.472)^3 = 0.895$. At a higher illustrative agreement level $p = 0.85$: $P_1 = \binom{5}{4}(0.85)^4(0.15) + (0.85)^5 = 0.39151 + 0.44370 = 0.835$, yielding $P(\text{consensus in } 3) = 1 - (0.165)^3 = 0.996$.
\end{theorem}

\noindent See Limitations (Section~9) for the independence assumption this bound relies on.

\begin{theorem}[Diminishing Returns Termination]
This is also a conditional/parametric result: \emph{if} information gain follows geometric decay $E[\Delta L^{(r)}] = \Delta L^{(1)} \cdot \gamma^{r-1}$, \emph{then} termination (three consecutive rounds with gain $< \epsilon$) occurs at $r^* = 1 + \log(\epsilon / \Delta L^{(1)}) / \log \gamma$. For illustrative values ($\Delta L^{(1)} = 10$, $\gamma = 0.4$, $\epsilon = 1$): $r^* \approx 3.5$, aligning with the hard circuit breaker. As with Proposition~10, $\gamma$ is a modeling assumption illustrated with an example value, not a quantity this paper estimates from any run.
\end{theorem}

\begin{figure}[t]
\centering
\includegraphics[width=0.75\textwidth]{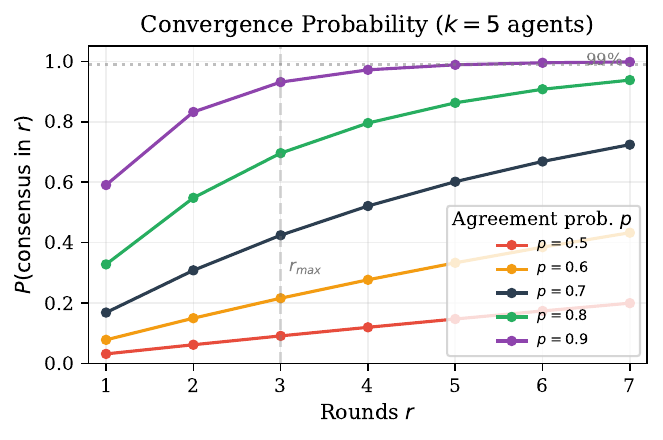}
\caption{Supermajority convergence probability as predicted by Theorem~12's formula, plotted across illustrative agreement probabilities $p$; not measured experimentally. At $p=0.7$: 89.5\% within 3 rounds; at the higher illustrative calibration point $p=0.85$: 99.6\%.}
\end{figure}

\begin{theorem}[Depth Bound]
For root goal complexity $\kappa(G_0)$ and uniform branching $n$:
\begin{equation}
d_{max} = \left\lceil \frac{\log \kappa(G_0)}{\log n} \right\rceil
\end{equation}
Beyond $d_{max}$, marginal value $\Delta V(d) \propto (1-q)^d$ decays exponentially while cost $O(n^d)$ grows, yielding a unique ROI-optimal depth $d^*$.
\end{theorem}

We operationalize $\kappa(G_0)$ as the number of MECE (Mutually Exclusive, Collectively Exhaustive) sub-goals generated by the Decision Maker in a single decomposition pass at temperature 0. As an illustration (not a measured value --- see Appendix~A), the prediction-market case study's root goal decomposes into 3 strategic sub-goals, each with 2--3 tactical children ($\kappa \approx 9$, $d_{max} = 2$ at $n = 3$). This approximation is model-dependent but reproducible given fixed prompts.

\paragraph{Property~7, revisited: the cost of $B_{up} = 0$.} Property~7 states that upward communication is zero in normal operation. This is a real architectural limitation, not only a cost saving: it means no information learned lower in the recursion tree ever propagates back up except through the narrow, deliberately coarse CONTRACT\_BREACH channel. A child node that discovers a nuanced, non-breach-level insight about its sub-goal has no formal channel to inform the parent's broader strategy. We return to this in Limitations (Section~9).

\section{Bayesian Inter-Layer Escalation}

Standard binary escalation (normal vs.\ breach) is too coarse for nuanced inter-layer learning. We introduce a continuous concern accumulator with categorical emission.

\begin{definition}[Concern Accumulator]
A leaky integrator: $c(t+1) = (1-\lambda)c(t) + \lambda e(t)$ where $e(t) \in [0,1]$ is evidence against a parent assumption at time $t$. Emits categorical signals: NOMINAL ($c < \theta_1$), CONCERN ($\theta_1 \leq c < \theta_2$), BREACH ($c \geq \theta_2$).
\end{definition}

\begin{theorem}[Unification]
The parent updates a posterior belief over the validity of its assumption given child signals:
\begin{equation}
P(\text{valid} \mid \text{signals}) \propto \prod_j P(s_j \mid \text{valid}) \cdot P(\text{valid})
\end{equation}
The parent re-evaluates when $P(\text{valid}) < 1 - \eta_d$, unifying the hysteresis threshold $\eta_d$ with the escalation protocol under a single Bayesian posterior. This is a belief update over the parent's assumption --- it does not, by itself, specify or authorize any action that changes the Execution Contract (Definition~5); see Section~8 for what, if anything, happens next.
\end{theorem}

This continuous-internal, categorical-external design appears throughout FRACTAL: mode switching (continuous saturation $\to$ discrete switch), consensus (continuous agreement $\to$ discrete termination), and depth allocation (continuous value decay $\to$ discrete stop/recurse decision).

\begin{figure}[t]
\centering
\includegraphics[width=0.75\textwidth]{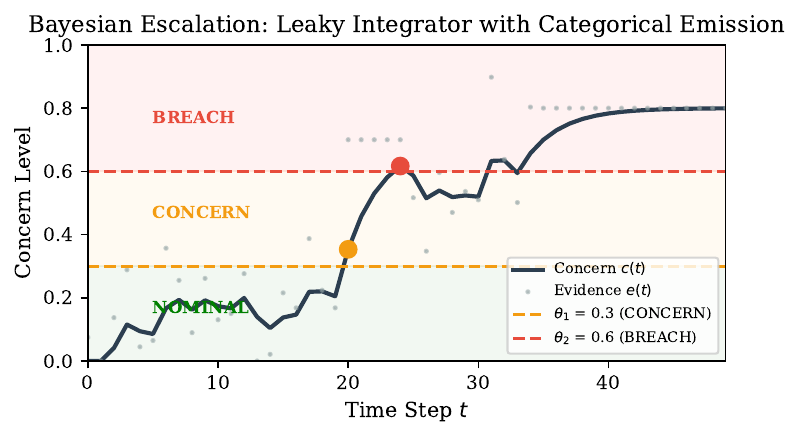}
\caption{An illustrative trajectory of the concern accumulator (Definition~15) under example evidence $e(t)$, not data from any run. The concern accumulator smooths noisy evidence into a continuous signal, crossing categorical thresholds for CONCERN and BREACH emission.}
\end{figure}

\section{Compatibility with Bounded Self-Improvement}

FRACTAL's architecture connects to the self-improving AI literature at three levels (Table~2). We state plainly, before describing these connections, what is and is not implemented: the Meta-Observer $M_d$ (Definition~1) is read-only, producing an audit trace $R_{audit}$ but never an action. Theorem~16 updates a posterior belief about a parent assumption's validity; it does not specify a mechanism by which that belief update becomes a modification to an Execution Contract's parameters ($\Theta, \tau_{exp}, G, T_{child}, P_{esc}, F_{cls}$). No component of the formal model as specified in Sections~3--7 performs such a modification. The connections below therefore describe what this architecture is \emph{compatible with extending to}, and what formal constraints such an extension would need to respect --- not a currently implemented self-modification capability.

\begin{table}[h]
\centering
\begin{tabular}{@{}lll@{}}
\toprule
Paradigm & System & FRACTAL analogue \\
\midrule
Inference-time refinement & Self-Refine \citep{madaan2023selfrefine} & Within-node convergent phase \\
Episodic memory & Reflexion \citep{shinn2023reflexion} & Knowledge Ledger persistence \\
Meta-agent search & ADAS \citep{hu2024adas} & Possible future direction$^*$ \\
Evolutionary generation & EvoAgent \citep{yuan2025evoagent} & Ensemble persona selection \\
Skill accumulation & Voyager \citep{wang2023voyager} & Ledger of validated strategies \\
Program search & FunSearch \citep{romeraparedes2024mathematical} & Falsification-guided refinement \\
Textual backpropagation & ANN \citep{ma2025selfevolving} & Possible future direction$^*$ \\
\bottomrule
\end{tabular}
\caption{Self-improvement paradigms and FRACTAL's architectural analogues (compatibility, not implementation). $^*$These two rows describe directions the architecture is compatible with extending toward, not a mechanism implemented in this paper; see the discussion below.}
\end{table}

\paragraph{Level 1: Within-Node (Reflexion-like).} The convergent phase critiques the divergent phase's proposals --- but unlike Self-Refine \citep{madaan2023selfrefine}, the critic and generator are architecturally separated (different personas, different information payloads), which is consistent with the broader finding that LLMs struggle with intrinsic self-correction absent external feedback \citep{huang2024cannot}.

\paragraph{Level 2: Cross-Layer (compatible with an ADAS-like extension).} A future extension could have parents observe child performance via concern signals and evolve contract specifications; no such mechanism exists in this paper's formal model. ADAS \citep{hu2024adas} argues code-based agent search spaces are Turing-complete; any such extension to FRACTAL would need to maintain the Self-Similarity Invariant --- modifying only parameters (personas, thresholds), never structure --- to preserve the composability Theorems~6, 12, and 14 rely on.

\paragraph{Level 3: System-Level.} Any future extension that modifies architectural parameters (diversity thresholds, dwell times) would, by this architecture's design intent, require Sovereign Override --- human approval --- before taking effect. Combined with the Self-Similarity Invariant and communication asymmetry (Property~7), these three constraints are the bounds any future bounded-self-improvement extension to FRACTAL would need to respect, addressing the recursive self-modification concern discussed broadly in the AI safety literature \citep{bostrom2014superintelligence}.

\begin{property}[Bounds on a possible self-improvement extension]
Any future contract-modification extension to this architecture would be constrained by: (1) it cannot modify termination conditions (Self-Similarity Invariant); (2) it cannot propagate changes upward without going through the escalation channel (communication asymmetry, Property~7); (3) it cannot alter system-level parameters without human approval (Sovereign Override). These are constraints on a possible extension, not a description of a capability implemented in this paper.
\end{property}

\section{Inductive Bias and Limitations}

\paragraph{Inductive bias.} FRACTAL's contract compression forces sub-goals to be independently specified, biasing outputs toward hierarchical, progressive-refinement structures. This is appropriate for nested-uncertainty domains (strategy, research, diagnosis) but potentially harmful for tasks requiring dense lateral coordination or flat parallel structure.

\paragraph{Single-LLM diversity.} When all personas share the same LLM, diversity arises from prompt conditioning and information asymmetry, not genuine cognitive difference. This provides structural value (mode separation, critique forcing) but not the information-theoretic diversity value of genuinely independent systems.

\paragraph{Formal guarantee scope.} Convergence bounds (Theorem~12) assume independent per-round agreement --- violated in practice because (a) agents share training data, (b) the Diversity Constraint creates anti-correlated biases by design, and (c) the shared Knowledge Ledger introduces round-to-round dependence. We report the bound at two calibration points ($p = 0.7$ and $p = 0.85$) as an illustrative model of convergence speed under idealized assumptions, not as a measured property of any implementation; we assume independence as a simplifying first-order model. A proof under correlated agents (using mixing-time arguments from Markov chain theory) is an important direction for future work.

\paragraph{The cost of zero upward communication.} As noted in Section~6, Property~7's $B_{up} = 0$ constraint means no information learned lower in the recursion tree propagates back up except through the coarse CONTRACT\_BREACH channel. This is a real architectural limitation: a child node's nuanced, sub-breach-level discoveries about its sub-goal have no formal channel back to the parent, and are lost once the child node terminates. Whether this cost is acceptable depends on the task --- strategy decomposition tasks where sub-goal insights rarely need to reshape the parent goal tolerate it better than tasks with tight cross-level feedback loops.

\paragraph{Diversity-threshold computability.} Definition~4's diversity constraint depends on a cognitive bias vector $\beta_i \in \mathbb{R}^{|B|}$ for each persona, but this paper does not specify how $\beta_i$ is computed from a persona's prompt or behavior in a concrete implementation. One candidate approach --- embedding each persona's system prompt with a named embedding model and computing cosine similarity on the resulting vectors --- is plausible but is our own untested design choice, not something we have implemented or validated here. We flag $\beta_i$'s computation as an open, unresolved point in the formal model rather than asserting a method we have not checked.

\paragraph{``Fractal'' and ``scalable'' as metaphor and theoretical property, not measured claims.} We use ``fractal'' as a metaphor for the node structure's recursive self-similarity (Definition~1, Remark~3), not a claim of measured fractal dimension or scale-invariant statistics. Likewise, any claim in this paper that the architecture is ``scalable'' refers to Theorem~14's depth bound and Corollary~8's cost bound holding theoretically as formal consequences of the self-similarity invariant --- not to demonstrated empirical scaling behavior at depth or ensemble sizes beyond what is illustrated in Appendix~A.

\paragraph{No empirical validation in this paper.} This paper presents the architecture and its formal communication, convergence, and depth-bound properties. Empirical validation of these properties --- and of the architecture's comparative performance against single-agent or flat-debate baselines --- is left to future work. The illustrative walkthrough (Sections~5--6) and the case study (Appendix~A) demonstrate that the formalism describes an executable protocol; they are not a quantitative evaluation and report no pass/fail rates, judge scores, or token measurements from any run.

\emph{A related empirical investigation of some of this architecture's components appears in concurrent work; forthcoming empirical validation of the properties characterized here is left to future work.}

\section{Ethics and Broader Impact}

This paper presents a formal multi-agent orchestration architecture and its communication, convergence, and depth-bound properties. No human-subjects data, personal data, or deployed system was used in this work; the illustrative walkthrough in Sections~1, 5, and 6 and the case study in Appendix~A are hypothetical, not records of real deployments. The architecture's compatibility with bounded self-improvement (Section~8) describes formal constraints on a possible future extension, not an implemented self-modifying system --- no autonomous modification of system behavior occurs in the model as specified here. We see no foreseeable harm specific to this paper's contribution beyond the general considerations applicable to any multi-agent LLM orchestration framework, including the general risk that formal bounds on a hypothetical future self-modification mechanism could be misread as already-implemented safety guarantees; we have tried to foreclose that misreading throughout Section~8 by stating plainly what is and is not implemented.

\section{Conclusion}

FRACTAL formalizes a self-similar recursive architecture for multi-agent orchestration: the same five-component node structure recurs at every depth, and communication, convergence, and depth bounds derived once hold at every layer by construction. The two-phase communication protocol reduces intra-node cost from quadratic to linear in ensemble size under a stated fixed-evaluator assumption; the two-mode orchestration criterion separates generative and evaluative phases via a measurable saturation threshold; the supermajority convergence guarantee and depth bound give explicit, checkable probability and complexity expressions; and the Bayesian escalation mechanism unifies hysteresis and contract-breach signaling under one posterior update. We have illustrated the architecture's execution on a qualitative walkthrough and a multi-layer case study, and have been explicit throughout about what is and is not proven, what is and is not implemented (particularly regarding self-improvement, Section~8), and what empirical questions --- comparative performance against single-agent or flat-debate baselines, chief among them --- this paper leaves open for future work.

\bibliographystyle{plainnat}
\bibliography{arxiv_references}

\appendix

\section{Multi-Layer Case Study: Prediction Market System}

We applied FRACTAL (depth=2, $n=3$) to designing a prediction market mispricing detection system --- a task requiring domain narrowing (Layer~0), signal detection (Layer~1A), mathematical formulation (Layer~1B), and data infrastructure (Layer~1C). This case study, like the running example in Sections~5--6, is an illustrative walkthrough of protocol execution, not a quantitative evaluation; it reports no pass/fail rates, judge scores, or token measurements.

\paragraph{Layer 0 (4 personas).} The Frame Challenger proposed ``market making'' as an alternative to ``mispricing detection'' --- rejected as standalone but absorbed as a hybrid element (providing liquidity to validate edge estimates). This illustrates the Frame-Challenger-as-regularizer mechanism: the rejected frame contributes a component to the accepted architecture.

\paragraph{Layer 1B (5 personas, $d=1$).} The Convergent phase discovered that critiques naturally clustered into three orthogonal groups (data freshness, model reliability, execution latency) --- revealing a progressive validation architecture invisible during proposal generation. This demonstrates stress-test-induced decomposition: the evaluation phase discovers problem structure that the generative phase does not surface on its own.

\paragraph{Inter-layer dynamics.} Child node 1B raised a CONCERN signal (leaky integrator crossing $\theta_1$) when it discovered that a parent assumption (Reverse Line Movement as primary signal) was invalid for the proposed infrastructure's 15-minute update frequency. The parent's Knowledge Ledger was updated at its next epoch; the child reprioritized within its own scope. This illustrates the escalation mechanism of Section~7 functioning as designed within this walkthrough: bottom-up signaling through the single CONCERN channel, without violating the temporal clock separation between layers.

This case study also grounds Theorem~14's worked $\kappa(G_0)$ example in Section~6: the root goal here decomposes into 3 strategic sub-goals (Layer~0), each with 2--3 tactical children at Layer~1 ($\kappa \approx 9$, $d_{max}=2$ at $n=3$ under the operationalization given in Section~6) --- illustrative of the operationalized complexity measure, not a value measured across any benchmark.

\section{Proof Sketches}

\paragraph{Theorem 6 (Communication Bound).} Across $r$ divergent rounds, each of $k$ agents generates one proposal per round at $T_p$ tokens: $r \cdot k \cdot T_p$ total. After the mode switch, the $m=2$ fixed evaluators each produce $c$ critiques per proposal at $T_s$ tokens: $m \cdot c \cdot T_s$ total. One synthesis step costs $T_{synth}$. Total: $rkT_p + mcT_s + T_{synth}$ --- linear in both $k$ and $r$. Free-form debate: each of $\binom{k}{2}$ pairs exchanges $r$ messages at $T_{resp}$ each, plus $k \cdot r$ turns at $T_{turn}$ --- quadratic in $k$, as shown in Equation~2.

\paragraph{Theorem 12 (Convergence).} Supermajority for $k=5$ requires $\geq 4$ of 5 agents agreeing: $P_1 = \binom{5}{4}p^4(1-p) + p^5 = 5p^4(1-p) + p^5$. At $p=0.7$: $P_1 = 5(0.7)^4(0.3) + (0.7)^5 = 5(0.2401)(0.3) + 0.16807 = 0.36015 + 0.16807 = 0.52822 \approx 0.528$. Over 3 rounds: $1-(1-0.528)^3 = 1-(0.472)^3 = 1-0.10504 = 0.895$. At $p=0.85$: $P_1 = 5(0.85)^4(0.15) + (0.85)^5 = 5(0.52200625)(0.15) + 0.4437053125 = 0.39151 + 0.44371 = 0.83521 \approx 0.835$. Over 3 rounds: $1-(1-0.835)^3 = 1-(0.165)^3 = 1-0.00447 = 0.99553 \approx 0.996$. We report $p=0.7$ and $p=0.85$ as two illustrative calibration points spanning the moderate-contentiousness regime stated in Theorem~12; neither is derived from a measurement of any implementation (see Limitations, Section~9).

\paragraph{Theorem 14 (Depth Bound).} At each layer, goal complexity divides by branching factor $n$: $\kappa_d = \kappa(G_0)/n^d$. Recursion bottoms out when $\kappa_d \leq 1$ (atomic action): $n^{d_{max}} \geq \kappa(G_0)$, giving $d_{max} = \lceil \log_n \kappa(G_0) \rceil$.

\end{document}